\documentclass{article}
\usepackage{spconf,amsmath,amssymb,graphicx,hyperref,bbm}
\usepackage{booktabs}
\usepackage{amsthm}
\usepackage{algorithm}
\usepackage{algpseudocode}
\usepackage[bottom]{footmisc}
\usepackage{bm}
\newtheorem{assumption}{Assumption}

\newtheorem{definition}{Definition}

\newtheorem{lemma}{Lemma}

\newtheorem{theorem}{Theorem}

\newcommand{\OPT}{\mathrm{OPT}}
\newcommand{\OPTLP}{\mathrm{OPT}_{\mathrm{LP}}}
\newcommand{\REW}{\mathrm{REW}}
\newcommand{\regret}{\mathrm{regret}}

\title{AFA-BANDIT: PROVABLY NEAR-OPTIMAL ONLINE MULTI-FEATURE CLASSIFICATION UNDER BUDGET CONSTRAINTS}

\twoauthors
 {AbdAlRahman Odeh, Teng-Hui Huang
 }
	{The University of Sydney\\
	School of Electrical and Computer Engineering\\
	Sydney 2006, NSW, Australia}
 {Hesham El Gamal}
	{The University of Sydney\\
	Faculty of Engineering\\
	Sydney 2006, NSW, Australia}
\begin{document}
\ninept
\maketitle
\begin{abstract}
Active Feature Acquisition (AFA) is a classification problem in which an agent decides which costly features to acquire before predicting each sample's label. Unlike batch AFA, which trains a fixed policy and classifier offline on fully observed data, online AFA updates its predictor from revealed labels as samples arrive. Existing online methods either use deep reinforcement learning (RL) without performance guarantees or maximize cost-adjusted reward rather than enforce a global budget. We formulate online AFA as a combinatorial Bandits with Knapsacks (BwK) problem that couples acquisition and prediction. Unlike prior bandit-based AFA and classical BwK, our setting has combinatorial complexity, evolving rewards, a global budget, and structured side information. We obtain an improved regret upper bound over standard BwK bounds in this framework, leveraging a cardinality-aware confidence bound and the subset update structure. To avoid an exponentially large action space, we propose \emph{LP-Chain}, a variant that searches a cost-aware chain of feature subsets with a size that grows linearly with the number of features. While the regret upper bound is specific to the combinatorial framework, \emph{LP-Chain} empirically achieves comparable predictive performance. On synthetic data, \emph{LP-Chain} outperforms HEDGE-based BwK and deep RL-based online AFA baselines and scales favorably to more features.
\end{abstract}
\section{Introduction}
\label{sec:intro}

Real-world intelligent systems must make decisions with limited resources. Acquiring information reduces uncertainty at a cost, e.g., in battery energy, latency, computation, or capacity~\cite{10.1007/978-3-031-16431-6_56,Cole2019Active}. Active Feature Acquisition (AFA) studies a sequential, resource-constrained decision-making problem: at each round, the learner acquires a subset of features at a cost and predicts the sample's label; it then uses the revealed label and acquired features to update its predictor before moving on to a new sample~\cite{schütz2026afabenchgenericframeworkbenchmarking}.

In batch-learning AFA, an agent has access to a labeled training dataset in which all features are observed for every sample. The agent jointly learns an acquisition policy and a classifier subject to an acquisition cost constraint~\cite{saar2009active}. The policy and the classifier are then deployed at inference time for cost-efficient classification of unseen samples~\cite{Janisch_2020}. However, this standard train--test paradigm is impractical for many decision-making applications. An agent often observes only on-policy outcomes, leaving off-policy outcomes unobserved~\cite{valancius2023acquisition}. Moreover, when the environment is time-varying, an agent optimized during training suffers from generalization error at inference time~\cite{li2021towards}.

Online AFA addresses these limitations by continuously updating the acquisition policy as the agent interacts with the environment. In~\cite{beyer2020afaondatastream}, the authors proposed a method that maximizes potential merit (predictive utility per unit cost) over data streams, but compared it empirically only against a random acquisition policy in a limited request-or-not setting. More recent works have adopted neural-network-based reinforcement learning (RL) approaches, owing to their ability to approximate the complex acquisition--prediction relation in online AFA. Kachuee et al.~\cite{kachuee2019opportunisticlearningbudgetedcostsensitive} proposed a two-network architecture in which a Deep Q-Network (DQN) learns the acquisition policy by maximizing the marginal certainty gain per unit cost, while a feed-forward network is trained to maximize prediction accuracy. However, such highly non-linear black-box architectures generally do not provide performance guarantees. 

Toward provably efficient solvers for online AFA, a body of work has adopted a multi-armed bandit (MAB) approach. In~\cite{atan2021datadriven}, a two-step formulation was proposed in which first-step actions correspond to feature acquisitions and second-step actions correspond to predictions, with accuracy as the unknown reward. Although a regret analysis was provided, the acquisition cost is subtracted from the reward, so no global budget constraint is enforced. Rahbar et al.~\cite{rahbar2025costefficientonlinedecisionmaking} applied a combinatorial MAB to the acquisition of test outcomes generated by a binary hidden state. Because they focus on minimizing acquisition costs, their regret result does not apply to prediction performance under a budget constraint. 

MAB-based online AFA with a global budget constraint is closely related to Bandits with Knapsacks (BwK)~\cite{badanidiyuru2018bandits}. In BwK, each arm is associated with a random reward and random resource costs drawn from unknown distributions. Pulling an arm yields a sample of its reward and costs, and the agent aims to maximize the cumulative reward until the budget of any resource is exhausted. A well-known result is that BwK algorithms achieve an optimal regret bound relative to an optimal policy with full knowledge of the reward and cost distributions~\cite{slivkins2024introductionmultiarmedbandits}. 
However, to the best of our knowledge, no BwK-based framework exists for online AFA.

Our contributions are summarized as follows. First, we formulate the online AFA problem as a predictor-coupled combinatorial BwK problem, distinguishing our setting from existing bandit-based AFA formulations through its evolving rewards, structured feedback, and global acquisition budget. 
Second, in this combinatorial framework, 
we establish an upper bound on regret relative to the reward achieved by an optimal policy that knows the data-generating distributions, and show that it is tighter than a direct application of the classical BwK bound.
Third, to reduce the exponential complexity (i.e., the size of the power set), our \emph{LP-Chain} uses a cost-aware greedy incremental subset-construction step that builds a chain of candidate subsets whose size grows linearly with the number of features, and empirically attains comparable predictive performance to that of its combinatorial, provable counterpart. 
Finally, we empirically demonstrate that \emph{LP-Chain} achieves better predictive performance than a BwK baseline and the closest online AFA baseline, which is based on deep RL.


\section{Problem Formulation}
\label{sec:problem}

We consider online multiclass classification over $T$ rounds ($T$ is known). At round $t$, an example $(X_t,Y_t)$ is drawn from an unknown distribution $\mathcal{P}$, where $Y_t\in[K]$ is the label, and $X_t=(X_{t,1},\ldots,X_{t,V})$ is an observation consisting of $V$ features. We assume $\mathcal{P}$ is a parametric mixture distribution, and for each $y\in [K]$, $\bar{x}_{N,v|y}=(1/N)\sum_{i=1}^Nx_{i,v|y}$ is an unbiased estimator of the $y$-component parameter $\theta_y$. Such an assumption applies to a general class of distributions, e.g., mixtures of exponential-family distributions~\cite{titterington1985statistical}.

Feature $v$ has a known acquisition cost $c_v\geq 0$. Furthermore, we assume that a set $S_0\subseteq[V]$ of features is available at no cost, and define the action space as $\mathcal{A}:=\{s|(s\subseteq[V])\cap (S_0 \subseteq s)\}$, i.e., every acquisition includes the free features. The cost of an action $S\in\mathcal A$ is $c(S):=\sum_{v\in S}c_v$, where $c_v=0$ for all $v \in S_0$. Acquisitions are subject to a global budget $B$ over the $T$ rounds. 

Let $\mathcal H_{t-1}$ denote the information available before round $t$. At the beginning of
round $t$, $(X_t,Y_t)$ is drawn independently from $\mathcal P$, but
neither $X_t$ nor $Y_t$ is observed by the learner. Based on $\mathcal H_{t-1}$, the learner selects a feature subset $S_t\in\mathcal A$ and observes only $X_{t,S_t}$. It then predicts the label using its current predictor $h_{t,S_t}$ corresponding to $S_t$. For any feasible subset $S \in \mathcal{A}$, define its potential round-$t$ reward, evaluated using the predictor before its round-$t$ update, as
\begin{equation}
    r_t(S)
    :=
    \mathbbm{1}\!\left\{
        h_{t,S}(X_{t,S})=Y_t
    \right\}.
    \label{eq:potential_reward}
\end{equation}
After the prediction, the true label $Y_t$ is revealed, and the learner observes the reward $r_t(S_t)$ of its selected action. The acquired observations $X_{t,S_t}$ are retained and can be used by the learner arbitrarily. The round-$t$ conditional expected reward of action $S$ is
\begin{equation}
\mu_t(S)
:=
\mathbb E \left[
r_t(S)\mid\mathcal H_{t-1}
\right].
\label{eq:conditional_expected_reward}
\end{equation}
Note that $\mu_t(S)$ changes over time as the predictor $h_{t,S}$ is updated, and that it is non-atomic, i.e., $\mu_t(S)\neq\sum_{s\in S}\mu_t(s)$, in contrast to Semi-BwK~\cite{pmlr-v84-sankararaman18a}. For each $S\in\mathcal A$, let $h^*_S$ denote the optimal predictor given $X_S$ under $\mathcal P$, and define the oracle round-$t$ reward and its expectation as
\begin{equation}
r^*_t(S):=\mathbbm{1}\!\left\{h^*_S(X_{t,S})=Y_t\right\},
\quad
\mu^*(S):=\mathbb E_{\mathcal P}\!\left[r^*_t(S)\right].
\label{eq:oracle_reward}
\end{equation}
For the formulated problem, the optimal reward is achieved by a dynamic oracle that knows $\mathcal{P}$ and predicts with $h^*_S$:
\begin{align}
\OPT:=\sup_{\pi\in\Pi}\;
\mathbb{E}_{\mathcal{P},\pi}\!\left[\sum_{t=1}^{\tau_\pi} r^*_t(S_t)\right],
\label{eq:opt_adaptive}
\end{align}
where $\tau_\pi\leq T$ is a stopping time induced by the acquisition budget, and the feasible policy set is defined as $\Pi:=\left\{\pi:\sum_{t=1}^{\tau_{\pi}}c(S_t)\leq B\right\}$. Evaluating $\OPT$ is intractable since it requires enumerating all possible acquisitions over $T$ rounds and comparing the rewards of those that respect the budget constraint. Instead, following~\cite{badanidiyuru2018bandits}, a tractable benchmark is obtained by the optimal static oracle, whose reward is the value of the linear program (LP)
\begin{equation}
\begin{aligned}
\OPTLP:=
\max_{p\in\Delta(\mathcal A)}\; & T\sum_{S\in\mathcal A}p(S)\,\mu^*(S)\\
\mathrm{s.t.}\quad & T\sum_{S\in\mathcal A}p(S)\,c(S)\leq B,
\label{eq:relaxed_lp}
\end{aligned}
\end{equation}
where we denote the maximizer by $p^*$. Note that the feasible solution set $\Delta(\mathcal{A})$ is a simplex of dimension $|\mathcal{A}|-1$, where $|\mathcal{A}|=2^{V-|S_0|}$. The solution $p^*$ is a distribution over actions that satisfies the budget only on average, i.e., $B/T$ per round; hence, the above LP is a relaxation of the problem in~\eqref{eq:opt_adaptive}. It is well known that $\OPTLP\geq \OPT$~\cite[Lemma~3.1]{badanidiyuru2018bandits} since any feasible action sequence of~\eqref{eq:opt_adaptive} is a feasible solution of the relaxed LP~\eqref{eq:relaxed_lp}. Hence, denoting the reward of any algorithm by $\REW=\sum_{t=1}^\tau r_t(S_t)$, where $\tau\leq T$ is the stopping time, we have $\OPTLP-\mathbb{E}[\REW]\geq \OPT-\mathbb{E}[\REW]$, so it suffices to bound the regret with respect to $\OPTLP$.


\section{Proposed Method}
\label{sec:methods}
We propose solving online AFA as a BwK problem in which each feasible subset $S \in \mathcal{A}$ is an arm with cost $c(S)$. 

\subsection{Dimension-aware confidence bounds}
\label{sec:confidence}
Our algorithm first estimates an upper confidence bound (UCB~\cite{Auer2002}) on the expected reward of each subset. We derive the following confidence interval to account for drifting rewards due to the dependency on prediction parameters, following~\cite{kleinberg2008multi,babaioff2015dynamic}. With a slight abuse of notation, we write $\mu(\boldsymbol{\theta}_S):=\mu(S)$. For a fixed $S$ and the vector parameter $\boldsymbol{\theta}_S=\begin{bmatrix}
    \boldsymbol{\theta}_{S,1}^T&\cdots&\boldsymbol{\theta}^T_{S,K}
\end{bmatrix}^T$, the $L$-smooth reward implies $|\mu(\boldsymbol{\theta}_S)-\mu(\boldsymbol{\theta}^*_S)|\leq L|\boldsymbol{\theta}_S-\boldsymbol{\theta}^*_S|$. Since $\hat{\boldsymbol{\theta}}_S:=\begin{bmatrix}
\hat{\boldsymbol{\theta}}^T_{S,1}&\cdots&\hat{\boldsymbol{\theta}}^T_{S,K}
\end{bmatrix}^T$ is an unbiased estimator of $\boldsymbol{\theta}_S$, concentration bounds apply.
\begin{definition}[Lipschitz Descent Lemma~\cite{nesterov2013introductory}]
    Let $g:\mathcal{X} \to \mathbb{R}$ be differentiable and \textit{$L$-smooth} with Lipschitz constant $L>0$. Then,
    \begin{equation}
        g(x)\leq g(y)+\langle\nabla g(y),y-x\rangle+\frac{L}{2}\lVert y-x\rVert^2.
    \end{equation}
\end{definition}
\begin{lemma}[Lemma 3.1 of~\cite{srebro2010smoothness}]\label{lemma:self_bounding}
    Let $\mathcal{X}$ be a closed convex set in a normed vector space (well-defined $\lVert\cdot \rVert$). 
    If a function $f:\mathcal{X}\to \mathbb{R}_{+}\cup\{0\}$ is $L$-smooth and non-negative, then $\lVert \nabla f(x)\rVert_*\leq\sqrt{4Lf(x)}$,
where $\lVert\cdot \rVert_*$ denotes the dual norm.
\end{lemma}

\begin{lemma}[Dimension-aware confidence bound]\label{lemma:dimUcb}
Suppose $\mu_t$ is $L$-Lipschitz smooth. For an action $S\in\mathcal A$, let $d(S):=|S|$ denote the dimension of its observation vector. Define
\begin{equation}
\beta_t(S)
:=
\sqrt{
\frac{
d(S)+2\log(t+1)
}{
N_t^{\mathrm{rew}}(S)
}
}.
\label{eq:dimension_confidence}
\end{equation}
The corresponding upper confidence bound is
\begin{equation}
f_t(S)
:=\widehat{\mu}_t(S)+
\beta_t(S)\sqrt{L\cdot\widehat{\mu}_t(S)}+c_b\beta_t^2(S),
\label{eq:optimistic_score}
\end{equation}
where $c_b>0$, $N^{\mathrm{rew}}_t(S)=\sum_{a\supseteq S}N_t(a)$, and $N_t(a)$ denotes the number of times that action $a$ has been taken before round $t$.
\end{lemma}
\begin{proof}
    
    For a bounded and Lipschitz smooth function, $|\nabla \mu(\theta)|\leq 2\sqrt{L\mu(\theta)}$ (Lemma~\ref{lemma:self_bounding} and~\cite{srebro2010smoothness}).
    Under the high-probability event, consider the following inequalities:
    \begin{align}
        \lVert \mu(\theta)-\mu(\hat{\theta}_N)\rVert \leq&\lVert \frac{d\mu(\hat{\theta}_N)}{d\theta}\rVert \lVert\theta-\hat{\theta}_N\rVert +\frac{L}{2}\lVert\theta-\hat{\theta}_N\rVert^2\\
        \leq&\sqrt{4L\mu(\hat{\theta}_N)}\lVert\theta-\hat{\theta}_N\rVert+\frac{L}{2}\lVert\theta-\hat{\theta}_N\rVert^2\\
        \leq &\sqrt{L\frac{c_{rad}\cdot \mu(\hat{\theta}_N)}{N}}+{\frac{Lc_{rad}}{N}},
    \end{align}
    where we use the self-bounding property of $\mu(\theta)$ for the second inequality, and the last one follows from the Hoeffding inequality.
    
\end{proof}

\subsection{Acquisition policy and global budget constraint}
\label{sec:acquisition}
\subsubsection{Combinatorial: Primal-dual steps}\label{subsec:combinatorial}
Following~\cite{slivkins2024introductionmultiarmedbandits}, solving the dual problem of the LP~\eqref{eq:relaxed_lp} with a Lagrange multiplier reveals the primal-dual structure:
\begin{equation}
    \max_{i\in\mathcal{A}}\mu(i)+\min_{\lambda}\lambda\left[\frac{B}{T}-c(i)\right],\label{eq:primal_dual_minimax}
\end{equation}
where $\lambda \geq 0$ is a multiplier for the average budget constraint $B/T$. Based on this insight, we introduce a dual variable $\lambda_t\in[0,\lambda_{max}]$, and the round-$t$ acquisition decision (primal step) is
\begin{equation}
    S_t=\underset{S\subseteq\mathcal{A}}{\arg\max}\,f_t(S)-{\lambda}_tc(S).\label{eq:primal_action}
\end{equation}
After subtracting the cost $c(S_t)$ from the budget, and after observing the reward, the dual variable is updated (dual step) by Online Mirror Descent (OMD)~\cite{hazan2016introduction}:
\begin{equation}
    {\lambda}_{t+1}=\underset{\lambda\in[0,\lambda_{max}]}{\arg\min}\, \lambda\left[\frac{B}{T}-c(S_t)\right] +\frac{\eta_t}{2}\lVert \lambda-\lambda_t\rVert^2,\label{eq:dual_action}
\end{equation}
where $\eta_t=1/\sqrt{(t+1)}$ is a step size. As in the BwK literature, $\lambda$ balances over- and underspending of the budget through OMD updates, which in turn affects the primal step~\cite{NIPS2016_f3144cef}.

\subsubsection{LP-Chain: Greedy chain construction}\label{subsec:lp_chain}
Clearly, direct optimization over $\mathcal A$ would require searching among $2^{V-|S_0|}$ feature subsets, and therefore scales exponentially with the number of features. \emph{LP-Chain} avoids this complexity by constructing a chain containing only $V-|S_0|+1$ candidate actions.

\begin{definition}[Feature chain]
A feature chain at round $t$ is an ordered family
\begin{equation}
\mathcal C_t
:=
\{S_{t,0},S_{t,1},\ldots,S_{t,V-|S_0|}\}
\subseteq \mathcal A,
\label{eq:chain_family}
\end{equation}
where $S_{t,0}=S_0
\subset S_{t,1}
\subset \cdots
\subset S_{t,V-|S_0|}=[V]$ with each successive action adding exactly one feature: $|S_{t,j}\setminus S_{t,j-1}|=1$.
\end{definition}

The chain is constructed greedily using the optimistic rewards. Starting from $S_{t,0}=S_0$, at step $j$ \emph{LP-Chain} selects
\begin{equation}
    v_{t,j}
    \in\arg\max_{v\notin S_{t,j-1}}
    \frac{f_t(S_{t,j-1}\cup\{v\}) -f_t(S_{t,j-1})}
    {c_v},
\label{eq:chain_greedy}
\end{equation}
and sets $S_{t,j}
:=
S_{t,j-1}\cup\{v_{t,j}\},
\text{ for }
j=1,\ldots,V-|S_0|$.

After constructing the chain, the algorithm allocates the remaining budget over its candidate actions. Before round $t$, the remaining budget is
\begin{equation}
B_{\mathrm{rem},t}
:=
B
-
\sum_{i=1}^{t-1}c(S_i),
\end{equation}
and the per-round allowance over the remaining rounds is:
\begin{equation}
b_t
:=
\frac{B_{\mathrm{rem},t}}
{T-t+1}.
\end{equation}
\emph{LP-Chain} then obtains an acquisition distribution by solving
\begin{equation}
\begin{aligned}
p_t\in\arg\max_{p\in\Delta(\mathcal C_t)}
\sum_{S\in\mathcal C_t}p(S)f_t(S)
\quad \mathrm{s.t.}\quad
\sum_{S\in\mathcal C_t}p(S)c(S)
\leq b_t,
\end{aligned}
\label{eq:lp_chain}
\end{equation}
where $\Delta(\mathcal C_t)$ denotes the probability simplex over $\mathcal C_t$. The learner then samples $S_t\sim p_t$; if $c(S_t)>B_{\mathrm{rem},t}$, it instead plays the zero-cost action $S_0$, which guarantees $\sum_{t=1}^{T}c(S_t)\leq B$ almost surely. Thus, in contrast to the OMD dual update~\cite{hazan2016introduction}, \emph{LP-Chain} enforces the budget directly through the per-round allowance $b_t$, while rebuilding the chain at every round lets the acquisition policy adapt to the evolving predictor.

\subsection{Online learning and subset feedback}
\label{sec:online_learning}
This subsection applies to both settings in Section~\ref{sec:acquisition}. Suppose that action $S_t$ is selected at round $t$. After observing $X_{t,S_t}$, the predictor assigns an observation to its nearest estimated class mean: $h_{t,S}(x_S)
:=
\arg\min_{k\in[K]}
\left\|
x_S-\widehat{\boldsymbol{\theta}}_{t,k,S}
\right\|_2^2$. After the prediction, the environment returns $r_t(S_t)$, and the cost $c(S_t)$ is subtracted from the budget. Recall that $N_t^{\mathrm{rew}}(S)$ is the number of observed rewards for action $S$ before round $t$, and let $\widehat{\mu}_t(S)$ denote their reward estimate. The reward estimates are updated as
\begin{equation}
\widehat{\mu}_{t+1}(S)
=
\widehat{\mu}_t(S)
+
\frac{
r_t(S)-\widehat{\mu}_t(S)
}{
N_t^{\mathrm{rew}}(S)+1
},
\quad S\subseteq S_t,
\label{eq:reward_update}
\end{equation}
and $N^{\mathrm{rew}}_{t+1}(S)=N^{\mathrm{rew}}_t(S)+1$ for these subsets. Since every feature in any $S\subseteq S_t$ has been acquired, the learner can evaluate the reward of all such subsets, not only of $S_t$, using the predictors $h_{t,S}$ before any round-$t$ update. Following this, since the label $Y_t$ is also revealed, the learner updates its predictor accordingly. For each class $k$ and feature $v \in [V]$, let $N_t^{\mathrm{feat}}(k,v)$ be the number of observations of feature $v$ from class $k$ available before round $t$, and let $\widehat{\theta}_{t,k,v}$ be the corresponding empirical class mean. When $Y_t=k$ and $v\in S_t$, the class mean is updated according to
\begin{equation}
    \widehat{\theta}_{t+1,k,v}
    =
    \widehat{\theta}_{t,k,v}
    +
    \frac{
        X_{t,v}-\widehat{\theta}_{t,k,v}
    }{
        N_t^{\mathrm{feat}}(k,v)+1
    }.
    \label{eq:mean_update}
\end{equation}
The associated count is increased by one, while the estimates for unobserved features and other classes remain unchanged.

\section{Regret Analysis}
\label{sec:regret}

\begin{figure*}[t]
\centering
\includegraphics[width=0.85\textwidth]{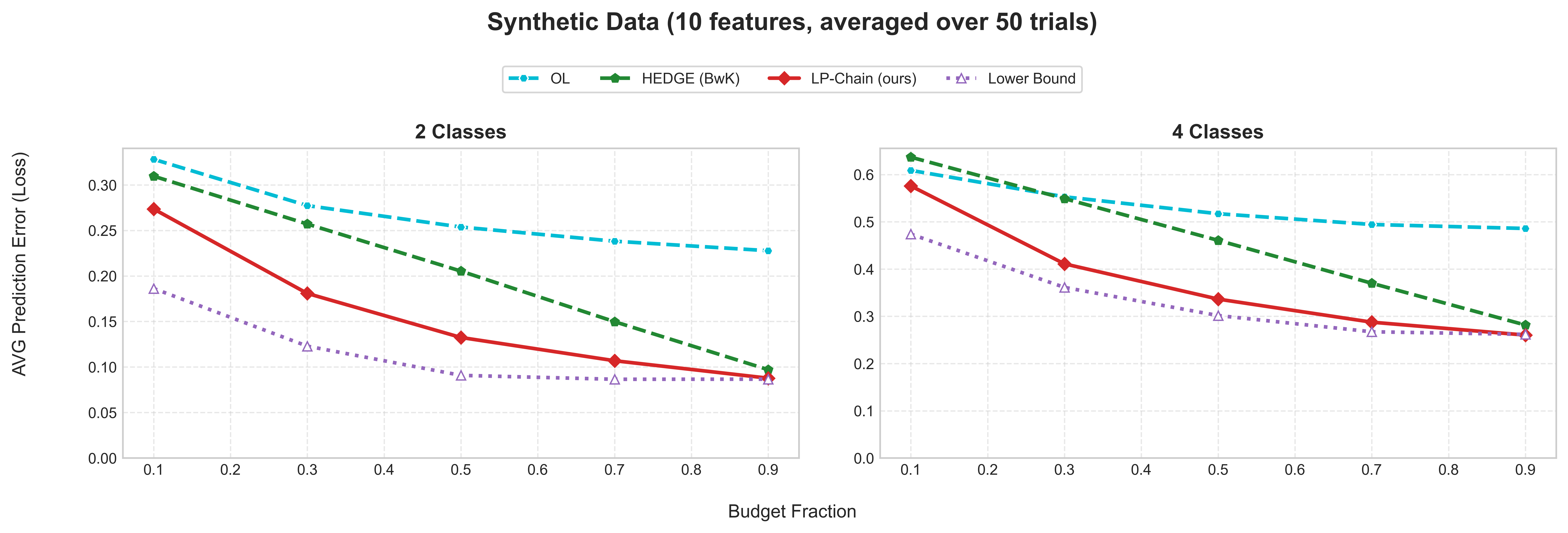}
\caption{Average prediction error versus budget fraction $q$ on synthetic data. OL: Opportunistic Learning~\cite{kachuee2019opportunisticlearningbudgetedcostsensitive}; HEDGE (BwK): HEDGE-based BwK policy~\cite{badanidiyuru2018bandits}; LP-Chain (ours): the proposed method; Lower Bound: prediction error of the full-action oracle $\OPTLP$. Lower is better.}
\label{fig:synthetic_results}
\end{figure*}
We derive a regret upper bound specific to the combinatorial setting in Section~\ref{subsec:combinatorial}. While the theoretical result does not apply to \emph{LP-Chain}, our empirical results show comparable predictive performance (see Section~\ref{sec:experiments}). The regret analysis for \emph{LP-Chain} is left for future work. As in the BwK literature, we measure regret against the LP relaxation rather than the dynamic oracle. With $\REW$ defined in Section~\ref{sec:problem}, the regret is given by
\begin{equation}
\regret(T) := \OPTLP - \mathbb{E}[\REW].
\label{eq:regret}
\end{equation} 

\begin{assumption}\label{assump:conditions}The following conditions hold:
    \begin{itemize}
        \item The observations $X_i,\forall i\in[V]$, the reward $r(s;\theta)\in[0,1]$, and the costs $c(S)\in[0,1]$ are bounded. 
        \item The expected reward $\mu(s;\theta), \forall s\subseteq[V],$ is $L$-Lipschitz smooth with respect to $\theta\in\Theta$.
    \end{itemize}
\end{assumption}

Our algorithm achieves a tighter regret bound, with an $\mathcal{O}(V^{\frac{1}{4}})$ factor improvement, than that obtained by directly applying BwK.
\begin{theorem}\label{THM:REGRET_UB}
Let the reward achieved by the optimal static policy be $\OPTLP$, and suppose Assumption~\ref{assump:conditions} holds. Then the regret \eqref{eq:regret} of the proposed algorithm satisfies
\begin{equation*}
\regret(T)\leq \mathcal{O}\left(\sqrt{\frac{2^{V}K}{\sqrt{V}}\OPTLP}+\OPTLP\sqrt{\frac{1}{B}}\right).
\end{equation*}
\end{theorem}
\begin{proof}[Proof sketch]
    Given $S_t$, the proposed algorithm evaluates the reward $r(s)$ for all the subsets $s\subseteq S_t$. Consider the action counters $N_A(s)$ and their total updates $N_t^{\mathrm{rew}}(s)$. By Lemma~\ref{lemma:dimUcb}, the key quantity in bounding the error between UCB and the actual reward is:
    \begin{equation}Q_t:=\sum_{a_t\in\mathcal{A}}\frac{\phi_{a_t}}{\sum_{q\supseteq a_t}\phi_{q}},\,\phi_{a_t}:=\frac{N_t(a_t)}{N^{\mathrm{rew}}_t(a_t)}. 
    \end{equation}
    The structure is equivalent to a reverse graded partially ordered set $G$~\cite{stanton_white_1986}. Using the result from~\cite{alon2017nonstochastic}, the error is bounded by the size of the maximum independent set of $G$. This maximum size is approximately $2^V/\sqrt{V}$ by Sperner's theorem~\cite{guichard2017introduction}. The update applies to each class $y\in[K]$, hence the factor $K$. See Appendix~\ref{appendix:theorem_proof} for details.
\end{proof}

\section{Experiments}
\label{sec:experiments}

\subsection{Experimental setup}

We evaluate \emph{LP-Chain} on synthetic Gaussian-mixture data with $V=10$ features and $K\in\{2,4\}$ classes; the Gaussian model only generates the data and is not assumed by the formulation, method, or analysis. Feature costs are heterogeneous and normalized so that $c([V])=1$, with one free feature. For a budget fraction $q\in[0,1]$, the total budget is $B=qTc([V])=qT$. Samples arrive as a single online stream, and results are averaged over 50 independent trials.\footnote{Code: \url{https://github.com/AbdAlRahman-Odeh-99/AMA-Bandit}.}

We compare \emph{LP-Chain} with the HEDGE-based BwK policy~\cite{badanidiyuru2018bandits} and with Opportunistic Learning (OL)~\cite{kachuee2019opportunisticlearningbudgetedcostsensitive} (see Section~\ref{sec:intro}), the AFA method closest to our setting, since it jointly trains its acquisition policy and predictor online. For a fair comparison, OL is run under the same global acquisition budget (see Appendix~\ref{sec:baseline_adaptation}). We also report $\OPTLP$, a full-action oracle that knows the true expected rewards $\mu^*(S)$ and solves~\eqref{eq:relaxed_lp} over the full action space $\mathcal{A}$. We report the average prediction error, i.e., the fraction of misclassified samples over the $T$ rounds (lower is better). Additional results, including comparisons on real-world data, are deferred to Appendix~\ref{sec:additional_experiments}.

\subsection{Synthetic data results}

Figure~\ref{fig:synthetic_results} reports the average prediction error across budget fractions. For both $K=2$ and $K=4$, the error decreases as the budget grows, since a larger budget allows more informative feature subsets to be acquired. Additionally, \emph{LP-Chain} attains lower average error than both HEDGE and OL at every budget fraction, and its gap to the lower bound narrows as the budget increases.

\subsection{Ablation: chain restriction vs.\ combinatorial search}

To evaluate the chain restriction, we compare \emph{LP-Chain} with \emph{Combinatorial} (Section~\ref{subsec:combinatorial}), which treats every feasible feature subset as an action. It searches a space of size $2^{V-|S_0|}$, whereas \emph{LP-Chain} considers only $V-|S_0|+1$ nested actions. We quantify the difference in predictive performance using the mean squared error (MSE) between their average errors across budget fractions. 

\begin{table}[t]
\centering
\caption{\textbf{Ablation:} Runtime (seconds) and prediction-error discrepancy.}
\label{tab:runtime}
\begin{tabular}{c|ccc}
\toprule
\textbf{$V$}
& \textbf{LP-Chain (s)}
& \textbf{Combinatorial (s)}
& \textbf{MSE $\downarrow$} \\
\midrule
5  & 6.71     & 4.38     & 0.00005 \\
10 & 19.91    & 10.01    & 0.00015 \\
15 & 126.72   & 118.60   & 0.00162 \\
20 & 2,461.93 & 5,397.70 & 0.00318 \\
\bottomrule
\end{tabular}
\end{table}

Table~\ref{tab:runtime} shows that the prediction-error discrepancy between \emph{LP-Chain} and \emph{Combinatorial} grows with $V$ but remains small, indicating that the chain restriction preserves most of the predictive performance of the full subset search. For $V \leq 15$, constructing the chain and solving its LP result in similar or slightly longer training times. At $V=20$, however, the cost of full enumeration dominates, and \emph{LP-Chain} is approximately $2.2\times$ faster. This computational advantage is expected to increase with $V$.

\section{Conclusion}
\label{sec:conclusion}

We cast online AFA as a predictor-coupled BwK problem, where each acquisition decision shapes both the current prediction and the data from which the predictor continues to learn. Exploiting the extra feedback from acquired subsets, we improve the regret bound on prediction error relative to the classical BwK bound by a factor of order $V^{\frac{1}{4}}$ in the combinatorial setting. To reduce complexity, our \emph{LP-Chain} constructs a cost-aware chain whose size grows linearly with the number of features. Empirically, it outperforms both BwK and deep RL-based online AFA baselines under the same budget.
In future work, we will extend the proposed method to bandit-label settings, where only a correct/incorrect outcome is returned instead of the true label. In this case, we can apply unbiased risk estimators from weakly supervised learning to update the prediction parameters~\cite{10.1093/nsr/nwx106}.

\bibliographystyle{IEEEbib}
\bibliography{refs}

\appendix

\section{Extended Evaluation}
\label{sec:additional_experiments}

This appendix supplements Section~\ref{sec:experiments}. We describe the comparison methods and real-world datasets, explain how Opportunistic Learning (OL) was run under the budgeted online protocol, and compare \emph{LP-Chain}'s subset reward update with an ablation that updates only the selected action. Results on real data are reported in Appendix~\ref{sec:real_results}.

\subsection{Baselines}
\label{sec:baselines_details}

We compare \emph{LP-Chain} with the HEDGE-based BwK policy and Opportunistic Learning (OL). A synthetic-only oracle provides an additional reference. The learned methods use the same dataset splits, feature costs, free feature, budget fractions, and total training acquisition budget $B=qT$. Their acquisition and prediction rules remain method-specific; in particular, OL also enforces a hard per-encounter spending cap of $q$, whereas \emph{LP-Chain} and HEDGE constrain spending over the horizon.

\textbf{HEDGE} is the full-action budgeted policy motivated by the BwK framework of Badanidiyuru et al.~\cite{badanidiyuru2018bandits}. Our implementation scores every feasible feature subset by its empirical reward estimate plus an upper-confidence bonus, divided by a resource-weighted effective cost. It updates the resource weights after each action. HEDGE and \emph{LP-Chain} use the same online nearest-class-mean predictor and, in the reported full-label-feedback runs, update empirical rewards for every subset contained in the acquired set. Thus, their comparison primarily tests full-action resource-weighted selection against \emph{LP-Chain}'s nested-chain selection and per-round LP allocation; it is not a selected-arm-only feedback comparison.

\textbf{OL} (Opportunistic Learning)~\cite{kachuee2019opportunisticlearningbudgetedcostsensitive} is an online cost-sensitive acquisition method designed for data streams. Its shared prediction/Q-network architecture combines a classifier (P-Net) with a feature-acquisition value function (Q-Net). For each sample, it begins with the free feature and sequentially acquires affordable, previously unobserved features. Its default exploitation rule ranks paid acquisition actions by Q-value divided by feature cost, with $\epsilon$-greedy exploration. Acquisition rewards measure the change in class confidence estimated by Monte Carlo dropout, rather than the sample's classification correctness. The method predicts after acquisition and then learns from the revealed label and stored transitions. Unlike HEDGE, OL uses its own neural classifier, so the OL comparison evaluates a complete acquisition-and-prediction method rather than an acquisition policy with \emph{LP-Chain}'s predictor held fixed.

$\bm{\OPTLP}$ is the synthetic-only oracle reference introduced in Section~\ref{sec:experiments}. The implementation uses the true synthetic class means to evaluate feasible subset actions on the generated training data, solves the full-action LP once under the budget, and samples from its fixed solution. The required true class means are unavailable for the real datasets, so this reference is reported only for synthetic experiments.

\subsection{Real-world datasets}
\label{sec:real_datasets}

We evaluate on six real-world benchmarks used in AFA studies and included in AFABench~\cite{schütz2026afabenchgenericframeworkbenchmarking}: three binary tabular datasets (CKD~\cite{chronic_kidney_disease_336}, BankMarketing~\cite{bank_marketing_222}, and PhysioNet~\cite{Goldberger2000}) and three multiclass datasets (Diabetes~\cite{cdc_nhanes_2018}, MNIST~\cite{726791}, and FashionMNIST~\cite{Xiao2017FashionMNISTAN}). Table~\ref{tab:real-datasets} gives the source dataset sizes and the metrics used in the OL comparison. We report F1 for the two imbalanced binary datasets, BankMarketing and PhysioNet, and accuracy for the other four datasets.

The comparison uses $V=16$ features per dataset, including one zero-cost feature. For tabular datasets, these are the first $16$ preprocessed features. MNIST and FashionMNIST images are first block-averaged to $4\times4$ features, then restricted to the first $16$. BankMarketing, PhysioNet, Diabetes, MNIST, and FashionMNIST are capped at 10{,}000 examples with balanced subsampling; CKD uses all 400 examples. Each seeded run uses a 60/20/20 training/validation/test split. Costs are normalized so that the full set has cost one, and the experiments use $q\in\{0.1,0.3,0.5,0.7,0.9\}$ with five seeds (42--46).

\begin{table}[t]
\centering
\caption{Source dataset statistics and metrics used in the OL comparison. The sample counts and feature counts shown are before the experiment-specific caps, image pooling, and reduction to $V=16$ features.}
\label{tab:real-datasets}
\begin{tabular}{lcccc}
\toprule
Dataset & \# Samples & \# Feat. & \# Cls & Metric \\
\midrule
CKD           & 400      & 24  & 2  & Accuracy \\
BankMarketing & 45{,}211 & 16  & 2  & F1 \\
PhysioNet     & 12{,}000 & 41  & 2  & F1 \\
Diabetes      & 92{,}062 & 45  & 3  & Accuracy \\
MNIST         & 60{,}000 & 784 & 10 & Accuracy \\
FashionMNIST  & 60{,}000 & 784 & 10 & Accuracy \\
\bottomrule
\end{tabular}
\end{table}

\subsection{Adapting OL to the budgeted comparison}
\label{sec:baseline_adaptation}

\begin{figure*}[t]
\centering
\includegraphics[width=\textwidth]{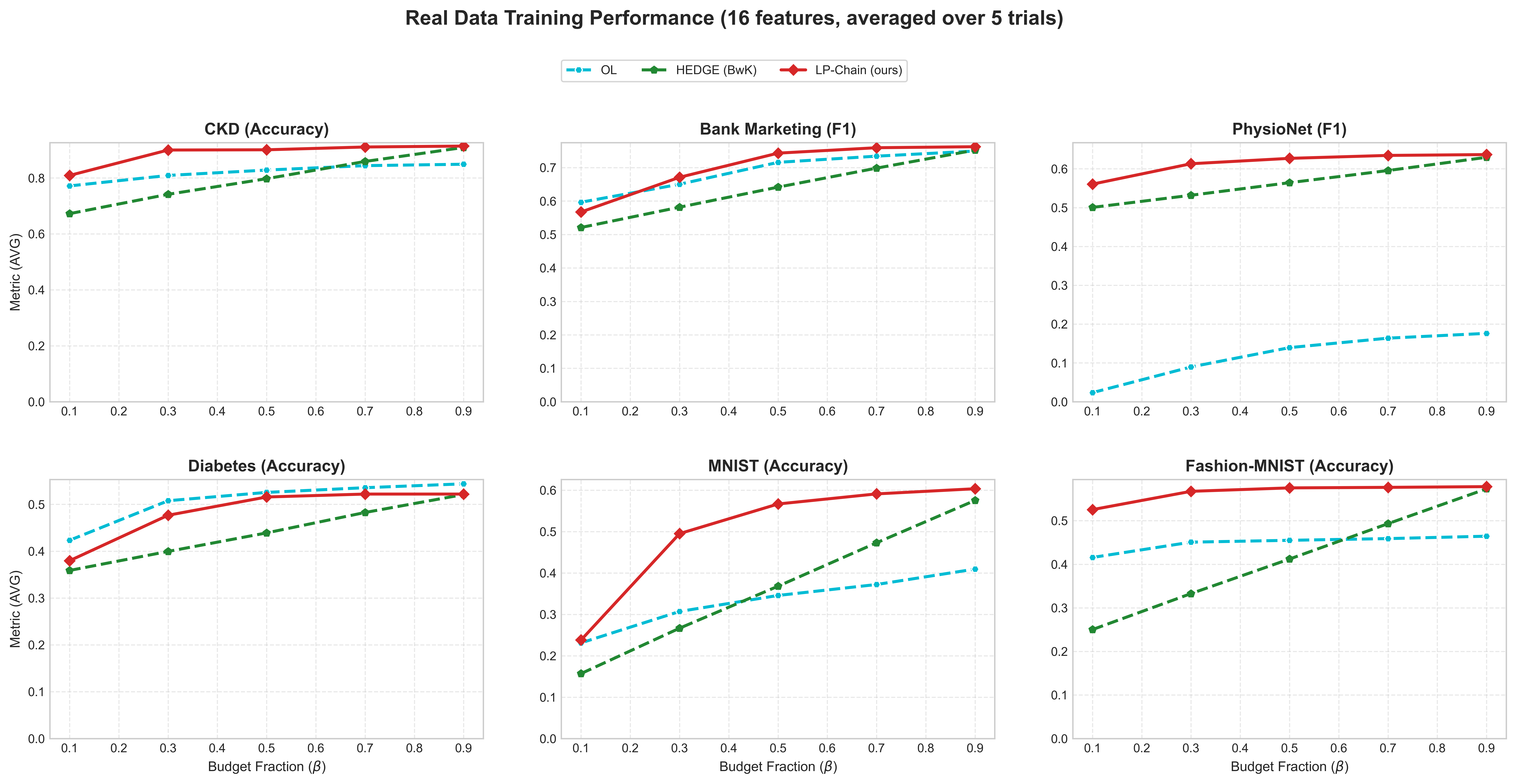}
\caption{Online training performance on six real-world datasets as a function of the budget fraction. BankMarketing and PhysioNet report F1; the remaining datasets report accuracy. Each curve averages five trials with $V=16$ features. Predictions are evaluated before the corresponding encounter updates; higher values are better.}
\label{fig:real}
\end{figure*}

OL already follows an online acquire--predict--update protocol, so we preserve its sequential acquisition process. Our implementation ports the authors' OL DQN demonstration to the shared datasets and feature representation. Each encounter starts with the free feature. At every acquisition step, actions for observed features and actions that exceed either the remaining per-encounter allowance $q$ or the remaining global training budget are masked out. With the reported defaults, greedy selection suppresses the stop action while a paid feature is affordable, although $\epsilon$-greedy exploration can stop earlier. The prediction is recorded before the current label is used for learning.

After each encounter, its acquisition transitions are added to a replay buffer. P-Net is updated from replay throughout training; Q-Net updates begin after a P-Net-only warm-up comprising the first 10\% of encounters. The reported configuration uses the Bayesian-L1 Monte Carlo dropout confidence-change reward, 100 dropout samples, $\epsilon$-greedy action selection, one replay update per encounter, and a discount factor of zero. A fresh model and budget pool are used for every seed and budget fraction. The standard run has as many encounters as training rows. Its class-balanced stream cycles through class-specific row pools, so this encounter count does not guarantee that every training row is seen exactly once. Validation and test inference use the frozen model, greedy acquisition, and separate budget pools.

The distinction between the methods matters for interpretation. \emph{LP-Chain} and HEDGE share a nearest-class-mean predictor and subset reward updates, whereas OL learns a neural classifier and acquisition Q-function from replay. OL's hard per-encounter cap also makes its feasible decisions more restrictive than those of the policies constrained only by the total budget. All reported online training metrics are based on predictions made before the corresponding encounter updates.

\subsection{Results on real data}
\label{sec:real_results}

\begin{figure*}[t]
\centering
\includegraphics[width=\textwidth]{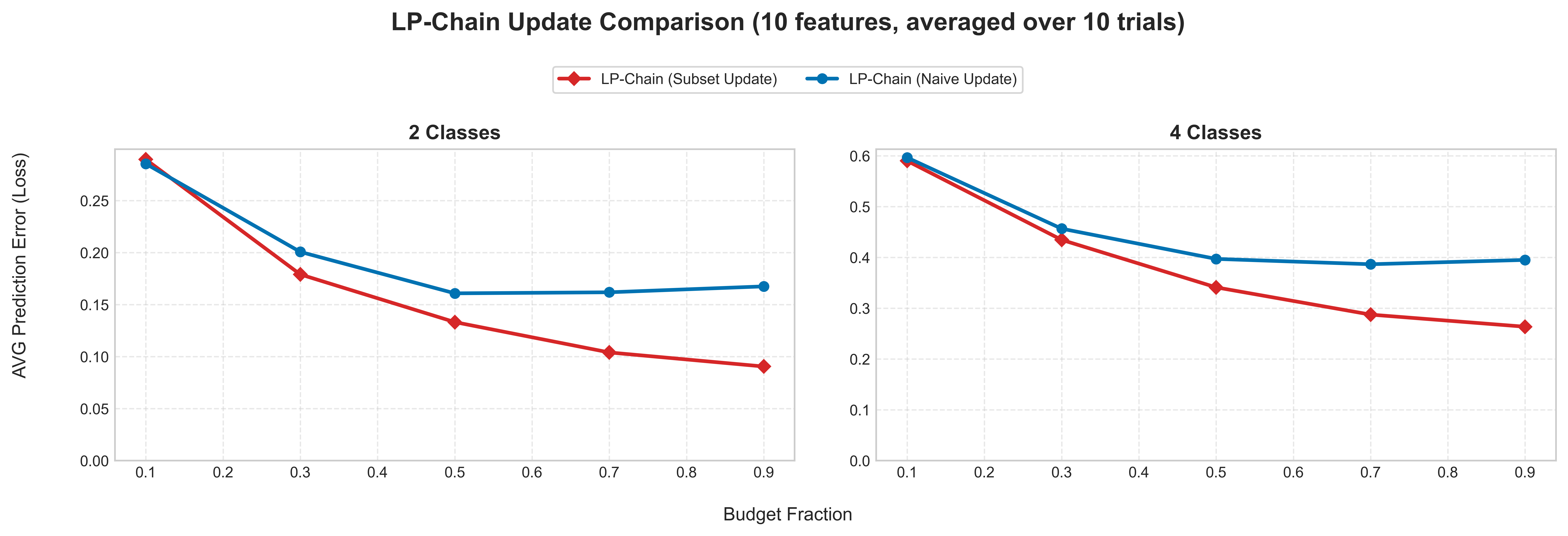}
\caption{Comparison of LP-Chain with Subset Update and Naive Update on synthetic data. Subset Update revises reward estimates for every feasible subset of the acquired action; Naive Update revises only the acquired action's reward estimate. Average online prediction error is shown against the budget fraction for $K\in\{2,4\}$ with $V=10$ features, over ten trials. Lower values indicate better performance.}
\label{fig:subset_vs_naive}
\end{figure*}

Figure~\ref{fig:real} shows that performance generally improves as the acquisition budget increases. \emph{LP-Chain} achieves the highest or approximately equal performance across the displayed budgets on CKD, BankMarketing, PhysioNet, MNIST, and FashionMNIST. Its advantage is particularly pronounced on MNIST. On BankMarketing, OL is slightly ahead at the smallest budget, while the methods converge at the largest budget. HEDGE also approaches \emph{LP-Chain} on several datasets as more features become affordable.

On Diabetes, OL achieves higher online training accuracy at every displayed budget, while \emph{LP-Chain} remains comparable and the gap narrows as the budget increases. OL's balanced training stream cycles through class-specific sample pools; when a pool is exhausted, it reshuffles and reuses its samples. This repeated exposure to minority-class examples, together with OL's replay-based predictor updates, likely contributes to its advantage on Diabetes. \emph{LP-Chain} does not use the same class-balanced stream, and adapting its acquisition and prediction updates to imbalanced data is left for future work.

The PhysioNet panel illustrates the importance of the evaluation metric for imbalanced binary data. OL's F1 remains substantially below that of \emph{LP-Chain} and HEDGE across the displayed budgets, despite its comparatively high training accuracy. These curves report online training performance and do not establish the same ordering on held-out test data.

\subsection{Subset versus Naive reward updates}
\label{sec:structural_update_ablation}

\emph{LP-Chain} maintains empirical reward estimates for feasible feature subsets. When it acquires $S_t$ and observes the label $Y_t$, every feasible $S\subseteq S_t$ can be evaluated using only the features already revealed. In the Subset Update version used for the main comparisons, we apply the current nearest-class-mean predictor to the features restricted to each such $S$ and update that subset's empirical reward estimate with the resulting correctness indicator $r_t(S)=\mathbbm{1}\{h_{t,S}(X_{t,S})=Y_t\}$ in~\eqref{eq:potential_reward}. These counterfactual subset scores are computed before the predictor is updated on the current sample. No additional features are acquired or charged to the budget.

To isolate the value of this subset feedback, we also consider \emph{LP-Chain} with a Naive Update. After initialization, Naive updates the empirical reward estimate of the acquired action $S_t$ using only the correctness indicator of the prediction actually made with $S_t$; it does not update the reward estimates of other subsets contained in $S_t$ in that round. During the forced full-feature initialization, the recorded prediction is uniform-random, so Naive skips the action-reward update rather than treating that prediction as the subset classifier's reward. The acquisition policy, nearest-class-mean predictor, class-label feedback, budget, and confidence-bound rule remain the same. Naive still uses the revealed $Y_t$ to update the predictor on the acquired features; only the information used to estimate action rewards differs.

Subset Update may provide many action-reward observations from one acquisition when $S_t$ is large, while Naive Update obtains one. Both versions remain online: they use only features observed on the current sample, the current predictor, and the label revealed after prediction. Comparing their online error at matched budgets measures the benefit of reusing this nested subset information in \emph{LP-Chain}'s reward estimates.

Figure~\ref{fig:subset_vs_naive} shows lower average prediction error with Subset Update across the displayed budgets in both class settings, apart from the smallest budget for two classes, where the errors are similar. The gap grows at higher budgets, particularly for four classes. This is consistent with the additional reward observations from contained subsets improving \emph{LP-Chain}'s acquisition decisions, without requiring extra feature acquisitions.

\section{Proof of Theorem~\ref{THM:REGRET_UB}}\label{appendix:theorem_proof}
The proof follows the steps of BwK and LinCBwK~\cite{badanidiyuru2018bandits,NIPS2016_f3144cef}; parts of the proof steps are included in the sequel for completeness.

\paragraph*{Reward Definition and Initialization}
We start with the definition of the reward achieved by the algorithm:
\begin{equation}
    \mathbb{E}[\REW]:=\sum_{t=1}^\tau \mu_t(S_t,\hat{\theta}_t),
\end{equation}
and the reward it would achieve under the optimistic upper bound:
\begin{equation}
    \mathbb{E}[\REW_{\mathrm{UCB}}]=\sum_{t=m+1}^\tau\tilde{\mu}_t(S_t,\hat{\theta}_t).
\end{equation}
We denote by $m$ the end of the initialization phase, during which all features are acquired; accordingly, the budget is adjusted to $B':=B-m$. Then, following~\cite{badanidiyuru2018bandits}, we have
\begin{align}
    \REW=&\sum_{t=1}^m\mu_t([V],\hat{\theta}_t)-\sum_{t=m+1}^\tau E_t^Ts_t+\REW_{\mathrm{UCB}}\nonumber\\
    &\geq \REW_{\mathrm{UCB}}-\sum_{t=m+1}^{\tau}E_t^Ts_t,
\end{align}
where we define $E_t:=\tilde{r}_t(h_t,S_t)-r_t(h_t,S_t)$, and $s_t\in\{e_j\}_{j\in[2^{V}-1]}$ denotes the selected feature subset.

\paragraph*{Decision and Confidence Interval}
By construction, under the clean event, the subset selected at round $t$ satisfies
\begin{equation}
    \tilde{\mu}_t(S_t,\hat{\theta}_t)-Z\lambda_t^T c(S_t)\geq \tilde{\mu}_t(S_t^*,\hat{\theta}_t)-Z\lambda_t^Tc(S_t^*).
\end{equation}
Rearranging, we have
\begin{equation}
    \tilde{\mu}_t(S_t,\hat{\theta}_t)\geq \tilde{\mu}_t(S^*_t,\hat{\theta}_t)-Z\lambda_t^T\left[c(S^*_t)-c(S_t)\right].
\end{equation}
Taking expectations over the data-generating distribution, we obtain
\begin{align}
   \tilde{\mu}_t(\pi_t,\hat{\theta}_t)\geq& \tilde{\mu}_t( \pi^*,\hat{\theta}_t)\nonumber\\
   &-Z\lambda_t^T\left(\frac{B'}{T}\boldsymbol{1}_d-\mathbb{E}[c(\pi_t)]\right).\label{eq:appendix_start_eq}
\end{align}
By construction, the confidence interval holds with probability at least $1-\delta$:
\begin{equation}
    {\mu}_t(\pi,\theta^*)\leq \widehat{\mu}_t(\pi,\hat{\theta}_t)+L_D\sqrt{\frac{|S|+\gamma\log{t}}{N_S(t)}}.
\end{equation}

Substituting the above into~\eqref{eq:appendix_start_eq}, we have
\begin{equation}
     \tilde{\mu}_t(\pi_t,\hat{\theta}_t)\geq  \mu_t(\pi^*,\theta^*)-Z\lambda_t^T\left[\frac{B'}{T}\boldsymbol{1}_d-\mathbb{E}[c(\pi_t)]\right] 
\end{equation}
Summing both sides from $t=m+1$ to a stopping time $\tau\leq T$:
\begin{align}
    \sum_{t=m+1}^\tau \tilde{\mu}_t(\pi_t,\hat{\theta}_t)\geq&\sum_{t=m+1}^\tau \mu_t(\pi^*,\theta^*)\\
    &-\sum_{t=m+1}^\tau Z\lambda_t^T\left[\frac{B'}{T}\boldsymbol{1}-\mathbb{E}[c(\pi_t)]\right].
\end{align}

\paragraph*{Online Mirror Descent Regret}
By the regret bound of online mirror descent, define $g_t(\lambda_t):=\lambda_t^T(\frac{B'}{T}\boldsymbol{1}-\mathbb{E}[c(\pi_t)])$, then we have:
\begin{align}
    {}&\sum_{t=m+1}^\tau \tilde{\mu}_t(\pi_t,\hat{\theta}_t)\nonumber\\
    &\geq (\tau-m) \OPTLP-Z\mathcal{O}(R_{\mathrm{OMD}})\nonumber\\
    &+Z\cdot \max_{\lambda\in\Delta_d}\lambda^T\left[\sum_{t=m+1}^\tau\mathbb{E}[c(\pi_t)]-\frac{B'(\tau-m)}{T}\boldsymbol{1}\right]
\end{align}

For the last term on the right-hand side, (i) if $\tau<T$, let $j^*$ denote the index of the depleted resource and set $\lambda'=e_{j^*}$; then we have: $\sum_{t=m+1}^\tau \mathbb{E}[c_{j^*}(\pi_t)]\geq B'$, so that:
\begin{equation}
\frac{B'(\tau-m)}{T}-\sum_{t=m+1}^\tau \mathbb{E}[c_{j^*}(\pi_t)]\leq 0.
\end{equation}
(ii) If $\tau=T$ (clean event), since $\sum_{t=m+1}^T\mathbb{E}[c_j(\pi_t)]\leq B',\forall j\in[d]$, the maximum occurs at $\lambda=\boldsymbol{0}$. Combining both cases, we have:
\begin{equation}
\max_{\lambda\in\Delta_d}\lambda^T\left[\sum_{t=m+1}^\tau \mathbb{E}[c(\pi_t)]-\frac{B'(\tau-m)}{T}\boldsymbol{1}_d\right]\geq \frac{B'}{T}(T-\tau+m).
\end{equation}

\paragraph*{Error Upper Bound}
Next, we focus on developing a tight upper bound on the error terms $E_t^Ts_t$. The key insight is that, at each round $t$, the reward of every subset of the acquired feature subset $\forall A_t\subseteq S_t$ can be evaluated and updated. Based on this insight, we have Lemma~\ref{lemma:subset_lemma}. The following graph-theoretic definitions are needed to prove the lemma~\cite{alon2017nonstochastic}.
\begin{definition}
For an undirected graph $G(X,E)$, a subset $U\subseteq X$ is an \textit{independent set of $G$} if no two $i,j\in U$ are connected by an edge, i.e., $(i,j)\notin E$.
\end{definition}
\begin{definition}
    An independent set $U$ is \textit{maximal} if no superset other than $U$ itself is an independent set of $G$. The \textit{independence number of $G$}, denoted by $\alpha(G)$, is the largest size of any $U$ of $G$.
\end{definition}

With these definitions, the key lemma used to prove the main result is borrowed from~\cite{alon2017nonstochastic}:
\begin{lemma}[Lemma 10 of~\cite{alon2017nonstochastic}]\label{lemma:alon_q_mas}
    Let $G=(X,D)$ be an acyclic directed graph with vertex set $X=[V]$ and arc set $D$. Then, for any distribution $p$ over $X$, we have
    \begin{equation}
        \sum_{i=1}^V\frac{p_i}{\sum_{j:j\rightarrow i}p_j}\leq \alpha(G),
    \end{equation}
    where $i\rightarrow j$ denotes a directed edge in $D$ from node $i$ to node $j$.
\end{lemma}
Based on the above results, we derive the following lemma, which bounds the regret of the proposed algorithm.

\begin{lemma}\label{lemma:subset_lemma}
    Let $E_t:=(\mu_t-r_t)^Tz_t$, where $\{\mu_t\}_{t\leq\tau}$ is a sequence obtained from the optimistic estimates of the proposed algorithm in Section~\ref{subsec:combinatorial} (combinatorial setting), whereas $\{r_t\}_{t\leq \tau}$ denotes the sequence of the obtained rewards. Then
    \begin{equation}
        \Bigg|\sum_{t=m+1}^\tau E_t\Bigg| \leq \mathcal{O}\left(\sqrt{\frac{K2^{V}\sum_{t=m+1}^\tau u_t}{\sqrt{V}}}\right).
    \end{equation}
    
\end{lemma}
\begin{proof}
    The steps follow Lemma 5.6 of~\cite{badanidiyuru2018bandits}, with a modification of the effective counters from the confidence bound. Since we consider the full-label-information setting, we can define the vector of subset counters for each class as:
    \begin{equation}
    \boldsymbol{n}(t):=\begin{bmatrix}
        \boldsymbol{n}^T_{1}(t)&\cdots&\boldsymbol{n}^T_{2^{V}-1}(t)
    \end{bmatrix}^T,
    \end{equation}
    where each $\boldsymbol{n}_s$ is a $K$-dimensional vector whose $j$-th element counts action $s$ when the label is $j\in[K]$.
    With this notation, we have $\boldsymbol{n}(t)=\sum_{s=1}^t\boldsymbol{z}_s$, where $\boldsymbol{z}_s$ is the elementary vector $\boldsymbol{e}_s(j)$ for action $s$ and label $j$.
    Moreover,
    \begin{align}
        {}&|(r_0-\bar{r}_t)^T\boldsymbol{n}(t)|\nonumber\\
        &\leq \sum_{j=1}^K\sum_{A\subseteq[V]_+}L_{rad}(\bar{r}_{t,A,j},\sum_{S\subseteq A}N_{S,j})N_{A,j}(t)\nonumber\\
        &\leq \sum_{j=1}^K\sum_{A\subseteq [V]_+}L_{rad}(\bar{r}_{t,A,j},\sum_{S\supseteq A}N_{S,j})N_{A,j}(t)\nonumber\\
        &=\sum_{j=1}^K\sum_{A_t\subseteq [V]_+}\sqrt{\frac{L_{rad}\bar{r}_{t,A,j}N^2_{A,j}(t)}{\sum_{S\supseteq A_t}N_{S,j}(t)}}+\frac{L_{rad} N_{A,j}(t)}{\sum_{S\supseteq A_t}N_{S,j}(t)},\label{eq:mas_derive_step1}
    \end{align}
    where $S\supseteq A$ ranges over all supersets of $A$, and we define $[V]_+=[V]\backslash\{\varnothing\}$. For  $A\subseteq[V]_+$ and $j\in[K]$, define:
    \begin{equation}
        p_{A,j}:=\frac{N_{A,j}}{\sum_{S\supseteq A}N_{S,j}},
    \end{equation}
    and
    \begin{equation}
        Q_{t,j}:=\sum_{A_t\subseteq[V]_+}\frac{N_{A_t,j}}{\sum_{S\supseteq A_t}N_{S,j}}=\sum_{A_t\subseteq[V]_+}\frac{p_{A_t,j}}{\sum_{S\supseteq A_t}p_{S,j}}.
    \end{equation}
    Then, by the Cauchy-Schwarz inequality, we obtain an upper bound on the first term of~\eqref{eq:mas_derive_step1}:
    \begin{align}
        {}&\sum_{A_t\subseteq[V]_+}\sqrt{\bar{r}_{t,A_t,j}N_{A_t,j}\frac{p_{A_t,j}}{
        \sum_{S\supseteq A_t}
        p_{S,j}}}\\
        &\leq \sqrt{\left(\sum_{A_t\subseteq[V]_+}\bar{r}_{t,A_t,j}N_{A_t,j}\right)Q_{t,j}}.\nonumber
    \end{align}
    Observe that, due to the update rule $N_{A,j}(t+1)=N_{A,j}(t)+\boldsymbol{1}\{S_t\supseteq A,Y_t=j\}$ for all $j\in[K]$, and that $[V]_+$ consists of all non-empty subsets, the update rule can therefore be described by a reversed graded partially ordered set~\cite{stanton_white_1986}, ordered by inclusion, and directed from the full set $\{[V]\}$ to the singletons $\{i\}_{i\in[V]}$. Define this graph as $G(X_{[V]_+},D_{\rightarrow})$, where
    \begin{align}
        X_{[V]_+}:=&\{s|s\subseteq [V]_+\},\\
        D_{\rightarrow}:=&\{i\rightarrow j| (|i|-|j|=1)\cap (j\subset i),\forall i\neq j\in X_{[V]_+}\}.
    \end{align}
    Note that the update graph is the same for all labels $j\in[K]$.
    Then, by Lemma~\ref{lemma:alon_q_mas}, $Q_t\leq \alpha(G)$. Following this, by Sperner's theorem~\cite[Chapter 1.7]{guichard2017introduction}, the maximum size of an independent set of $G$ is attained by the family of subsets whose cardinality is half that of the full set. Therefore,
    \begin{equation}
        \alpha(G)\leq \max\left\{{V\choose \lfloor V/2\rfloor},{V\choose \lceil V/2\rfloor}\right\}.
    \end{equation}
    Lastly, by Stirling's approximation $n!\approx\sqrt{2\pi n}(n/e)^n$, and assuming $V$ is even (without loss of generality), we have:
    \begin{equation}
        {V\choose  V/2}= \frac{V!}{(V/2)!(V/2)!}\approx 2^V\sqrt{\frac{2}{\pi V}}.
    \end{equation}
The dependence on $K$ follows from the Cauchy-Schwarz inequality over the classes:
\begin{equation}
    \sum_{j=1}^K\sqrt{L_{rad}R_j\frac{2^V}{\sqrt{V}}} \leq \sqrt{L_{rad}\left(\sum_{j=1}^KR_j\right)\left(\sum_{j=1}^K1\right)\frac{2^V}{\sqrt{V}}},
\end{equation}
where we define $R_j=\sum_{A_t\subseteq[V]_+}\bar{r}_{t,A_t,j}N_{A_t,j}$.
\end{proof}

\paragraph*{Overall Regret}
Combining the previous steps, we obtain
\begin{align}
    \mathbb{E}[\REW]\geq& \mathbb{E}[\REW_{\mathrm{UCB}}]-\sum_{t=m+1}^\tau E_t^Ts_t\nonumber\\
    \geq&(\tau-m)\OPTLP-Z\mathcal{O}(R_{\mathrm{OMD}})\nonumber\\
    &+\frac{ZB'}{T}(T-\tau+m)-\Bigg|\sum_{t=m+1}^\tau E_t^Ts_t\Bigg|.
\end{align}
By choosing $Z=\frac{\OPTLP}{B'}$, we get:
\begin{equation}
    \REW \geq \OPTLP -\mathcal{O}\left(\frac{\OPTLP}{B'}R_{\mathrm{OMD}}\right)-\Bigg|\sum_{t=m+1}^\tau E_t^Ts_t\Bigg|.\label{eq:appendix_regert_e_norm}
\end{equation}


Finally, with the time-varying step size $\eta_t=\mathcal{O}(\sqrt{\frac{1}{t}})$, the OMD regret is known to be $R_{\mathrm{OMD}}=\mathcal{O}(T)$~\cite{hazan2016introduction}. Combining this with Lemma~\ref{lemma:subset_lemma}, we obtain
\begin{align}
    \mathbb{E}[{\REW}]\geq& \OPTLP-\mathcal{O}\left(\frac{\OPTLP}{B'}\sqrt{T}+\sqrt{\frac{K2^V}{\sqrt{V}}\mathbb{E}[\REW]}\right)\\
    \geq& \OPTLP - \mathcal{O}\left(\OPTLP\sqrt{\frac{1}{B}}+\sqrt{\frac{K2^V}{\sqrt{V}}\OPTLP}\right),
\end{align}
where we use the assumption $B=\Omega(T)$ and set $d=1$, since we consider a single budget constraint.

\end{document}